%% file: log-CCV-ALT27-extended.tex
\documentclass[preprint]{alt2027}
\title[$\tilde{O}(\sqrt{T})$ Regret and Polylogarithmic Constraint Violation for COCO]
{$\tilde{O}(\sqrt{T})$ Regret and Polylogarithmic Constraint Violation for COCO}

\altauthor{%
 \Name{Dhruv Sarkar} \Email{dhruv.sarkar223@gmail.com}\\
 \addr MIT, Cambridge, USA
 \AND
 \Name{Abhishek Sinha} \Email{abhishek.sinha@tifr.res.in}\\
 \addr Tata Institute of Fundamental Research, Mumbai, India%
}

\usepackage{times}
\usepackage{booktabs}
\usepackage{nicefrac}
\usepackage{microtype}
\usepackage{mathtools}
\usepackage{mathrsfs}
\usepackage{dsfont}
\usepackage{bm}
\usepackage{color}
\usepackage[utf8]{inputenc} % allow utf-8 input
\usepackage[T1]{fontenc}    % use 8-bit T1 fonts
\usepackage{hyperref}       % hyperlinks
\usepackage{url}            % simple URL typesetting
\usepackage{booktabs}       % professional-quality tables
\usepackage{amsfonts}       % blackboard math symbols
\usepackage{nicefrac}       % compact symbols for 1/2, etc.
\usepackage{microtype}      % microtypography
\usepackage{setspace}
\usepackage{dsfont}
\usepackage{amsmath}
\usepackage{amssymb, bm}
\usepackage[dvipsnames,svgnames]{xcolor}
\usepackage{xcolor}
\usepackage{hyperref}
\usepackage{mathrsfs}
\usepackage{mathtools}
\usepackage{algorithm}
\usepackage{algorithmic}
\usepackage{graphicx}

\usepackage{nicefrac}
\usepackage{framed}
\newcommand{\X}{\mathcal X}
\newcommand{\R}{\mathbb R}
\newcommand{\Regret}{\mathsf{Regret}}
\newcommand{\ind}{\mathbf1}
\newcommand{\COCO}{$\mathsf{COCO}$}

\let\geq\geqslant
\let\leq\leqslant
\let\ge\geqslant
\let\le\leqslant

\mathcode`*="8000
{\catcode`*=\active
\gdef*{\star}}

\hypersetup{
    colorlinks=true,
    linkcolor=blue,
    urlcolor=red,
    citecolor=NavyBlue
}

\newtheorem{assumption}{Assumption}
\newtheorem{prelemma}{Lemma}
\newcommand{\CCV}{\mathsf{CCV}}

\newcounter{commentcounter}

\setcitestyle{authoryear,open={(},close={)}}

\begin{document}

\maketitle

\input{abstract}
\section{Introduction}
Constrained Online Convex Optimization (\COCO) concerns the problem of optimal sequential  decision-making in the presence of unknown, time-varying, and adversarial constraints \citep{sinha2024optimal, guo2022online, sarkar2026geometric}. It generalizes the classical Online Convex Optimization framework \citep{hazan2022introduction} by requiring the learner to simultaneously minimize both regret and long-term constraint violations. The constrained learning problem lies at the heart of a diverse range  of practical problems including AI safety, alignment, fairness, optimal resource allocation with budget constraints while maintaining the optimal performance.

In this model, at each round $t \in [T]$, an online learner selects an action $x_t$ from a convex decision set $\mathcal{X} \subset \mathbb{R}^d$ of diameter $D.$ After observing the learner's action $x_t$ for the current round, an adaptive adversary reveals a convex cost function $f_t: \mathcal{X} \mapsto \mathbb{R}$ and a convex constraint function $g_t: \mathcal{X} \mapsto \mathbb{R}.$\footnote{All our results and algorithms remain valid without any modification if the constraint functions are assumed to be only \emph{quasiconvex} instead of convex. Recall that a function $h$ is called quasiconvex if all its sublevel sets $\{h(x) \leq c\}$ are convex for every $c \in \mathbb{R}$ \citep{boyd}. Clearly, all convex functions are quasiconvex. Our analysis uses only the convexity of these sublevel sets.} The constraint function corresponds to the constraint $g_t(x) \leq 0.$ Consequently, the learner incurs a cost of $f_t(x_t)$ and a constraint violation of $\max(0, g_t(x_t)).$ Both the cost and constraint functions are assumed to be $G$-Lipschitz. The performance of the online policy is benchmarked against the set of feasible actions that satisfy \emph{all} $T$ constraints in the hindsight. Specifically, we define the feasible set $\mathcal{X}^*$ to be 
\begin{eqnarray} \label{feas-set}
	\mathcal{X}^*= \cap_{t=1}^T\{x \in \mathcal{X}: g_t (x) \leq 0 \}.
\end{eqnarray}
%We make the  \emph{uniform feasibility} assumption: there exists a fixed action 
Throughout the paper we make the following assumption:
\begin{assumption}[Common Feasibility] \label{feas-assump}
	 $\mathcal{X}^* \neq \emptyset.$
\end{assumption} 
We use the standard definition for $\Regret$ against the set of benchmarks $\mathcal{X}^*$ and Cumulative Constraint Violation ($\CCV$) for any online policy: 
\begin{eqnarray*}
	\Regret_T = \sum_{t=1}^T f_t(x_t) - \min_{x^* \in \mathcal{X}^*} \sum_{t=1}^T f_t(x^*), ~~\CCV_T = \sum_{t=1}^T \max\big(0, g_t(x_t)\big).
\end{eqnarray*} 

In hindsight, playing the best fixed action from the set $\mathcal{X}^*,$ which can be determined only at the end, incurs zero $\Regret$ and zero $\CCV$. The objective is to design an \emph{online} policy that achieves the smallest possible regret and CCV. We focus primarily on the dependence on \(T\), treating \(d\), \(G\), and \(D\) as fixed problem parameters. Without any additional assumptions (\emph{e.g.,} strong-convexity, smoothness, Slater condition or otherwise), the state-of-the-art algorithms for $\mathsf{COCO}$ for convex costs and constraints achieve $O(\sqrt{T})$ $\Regret$ and $O(\sqrt{T})$ $\CCV$ \citep{sarkar2026geometric, sinha2024optimal}. In this paper, we present an online learning algorithm that achieves $O(\sqrt{T \log T})$ $\Regret$ and $O((\log T)^2)$ $\CCV$, reducing the $\CCV$ from a polynomial to a polylogarithmic bound. The proposed algorithm is based on continuous-hedge on shrinking decision sets. 

Our main conceptual contribution is a continuous elimination principle for $\mathsf{COCO}$. A constraint violation eliminates a constant fraction of the Hedge probability mass, by Gr\"unbaum's inequality applied to the violated mean. While this probability-mass reduction need not correspond to a comparable reduction in uniform volume, we bridge the two through an adaptive learning rate and a potential combining the surviving volume with the learning rate. When the Hedge density is highly nonuniform, a violation forces a learning-rate decrease; once the density is approximately uniform, the same violation forces a constant fraction of volume to disappear. Thus potential decreases by a constant on every  round with $g_t(x_t)>\delta$, yielding polylogarithmic $\CCV$ while retaining near-square-root regret.

\subsection{Related work}\label{sec:related}

\paragraph{Online optimization with constraints:}
\citet{mannor2009online} initiated the study of online optimization with long-term constraints and time-varying constraint functions. Table \ref{gen-oco-review-table} summarizes several relevant prior works and their guarantees. 
\citet{mahdavi2012trading} studied online optimization with fixed and known constraints, motivated by avoiding the computational cost of projection at the expense of maintaining only long-term feasibility. Earlier works on COCO primarily employ primal-dual techniques and do not fully exploit the geometry of the nested feasible sets, leading to suboptimal $\CCV$ bounds. 
\citet{sinha2024optimal} presented a first-order policy achieving $O(\sqrt{T})$ regret and $O(\sqrt{T}\log T)$ $\CCV$ for \COCO~without additional assumptions. Their approach uses a Lyapunov function of the accumulated constraint violation to weight the constraint in a surrogate loss. 
\citet{pmlr-v313-supantha26a} further refined these results by establishing sublinear universal dynamic regret and $\CCV$ bounds without assuming a common feasible point. 
\citet{vaze2025osqrtt} presented a projection-based algorithm achieving $O(\sqrt{T})$ regret and an instance-dependent $\CCV$ bound. 
Other related papers consider special cases of COCO and its variants, including fixed constraints \citep{yi2021regret}, fixed objective functions \citep{balasundaram2026convex}, projection-free methods \citep{proj-free-Sinha26, proj-free-garber}, strongly convex costs \citep{guo2022online, sinha2024optimal}, stochastic constraints \citep{neely2017online}, and Slater's condition \citep{yu2020low}.

\paragraph{Regret--CCV tradeoffs:}
\citet{sinha2025beyond} constructed a Lyapunov-based surrogate cost function for the constrained experts problem and used a covering argument to extend the analysis to general convex functions. Suppressing fixed Lipschitz and diameter factors, their bounds for a parameter $\theta\in[0,1]$ are
\[
 \Regret_T=O\!\left(\sqrt{dT\log T}+T^\theta+d\log T\right),
 \qquad
 \CCV_T=O\!\left(dT^{1-\theta}\log^2T\right).
\]
At $\theta=1/2$, these bounds give $O_d(\sqrt{T\log T})$ regret and $O_d(\sqrt{T}\log^2T)$ $\CCV$. These guarantees are achieved by a variant of the adaptive Hedge algorithm whose complexity, after a covering-based discretization, scales as $O(T^d)$. In particular, $O((\log T)^2)$ $\CCV$ is achieved at the  expense of a trivial linear regret.
%For smooth and non-negative constraints, in the absence of costs, an Online Gradient Descent-based policy achieves $O(\sqrt{T})$ regret and $O(1)$ $\CCV$ \citep{sinha2025beyond}. 
In contrast, the present paper achieves $O((\log T)^2)$ $\CCV$ without trading-off the near-optimal regret.

\paragraph{Elimination and Continuous Hedge:}
\citet{sarkar2026geometric} established that Active Hedge with elimination achieves $O(\sqrt{T\log N})$ regret and $O(N)$ $\CCV$ for a finite set of $N$ experts. We retain its exponential weighting and permanent elimination of infeasible actions, and extend this approach to the continuous domain. 
The continuous forecaster proposed by \citet{krichene2015} uses a continuous Hedge algorithm with a non-increasing learning rate. We adapt this density to a changing convex support and couple it with Gr\"unbaum's inequality, which quantifies the probability mass removed by each violated constraint \citep{lovasz2007geometry}. 
The adaptive learning-rate schedule and the $\CCV$ potential connect this geometric argument to the Active Hedge algorithm of \citet{sarkar2026geometric}. 
Unlike the projection-based algorithms of \citet{sarkar2026geometric} and \citet{vaze2025osqrtt}, which play on the boundary of the current feasible set, our policy plays the mean of the Hedge distribution. This interior-point property is crucial: whenever the mean violates a constraint, Gr\"unbaum's inequality guarantees that a constant fraction of the Hedge probability mass lies outside the updated feasible set, which enables the improved $\CCV$ bound.  
 \begin{table*}[t]
%\hspace{-30pt}
  \title{Summary of the results for the constrained OCO problem}
  %\centering
  %\hspace{-40pt}
  \begin{tabular}{llllll}
    \toprule
    %\multicolumn{2}{c}{Part}                   \\
   % \cmidrule(r){1-2}
   \small { Reference}  & \small {Regret} & \small {CCV} & \small {Assumptions/Comments} \\
    \midrule
  %  a & b& c& d & e  \\
%      \small {\citet{mahdavi2012trading}}  & \small {$O(\sqrt{T})$} & \small {$O(T^{\nicefrac{3}{4}})$} & \small {Projection} & \small{Time-invariant constraints} \\
 %   \small {\citet{jenatton2016adaptive}}  & \small {$O(T^{\max(\beta, 1-\beta)})$} & \small {$O(T^{1-\beta/2})$} & \small {Projection} & \small{Time-invariant constraints} \\
    \small {\citet{pmlr-v70-sun17a}}  & \small {$O(\sqrt{T})$} & \small {$O(T^{\nicefrac{3}{4}})$}& & \small -  \\
    \small {\citet{neely2017online}}  & \small {$O(\sqrt{T})$} & \small {$O(\sqrt{T})$}  & \small {Slater condition} \\
    %    \small {\citet{yu2017online}}  & \small {$O(\sqrt{T})$} & \small {$O(\sqrt{T})$} & \small {Conv-OPT} & \small {Slater condition} \\
  %\small {\citet{yuan2018online}} & \small {$O(T^{\max(\beta, 1-\beta)})$} & \small {$O(T^{1-\beta/2})$ }  & \small {Projection} & \small{Time-invariant constraints} \\
      \small {\citet{yu2020low}}  & \small {$O(\sqrt{T})$} & \small {$O(1)$}  & \small {Slater \& Time-invariant constraints} \\
 % \citet{yu2017online} & Stochastic & $O(\sqrt{T})$ & $O(\sqrt{T})$& OGD+drift+penalty & Slater condition \\
%  \small {\citet{yi2021regret}} & \small {$O(T^{\max(\beta, 1-\beta)})$} & \small {$O(T^{(1-\beta)/2})$} & \small {Conv-OPT} & \small{Time-invariant constraints} \\ 
%  \small {\citet{yi2022regret}}  & \small {$O(T^{\beta})$} & \small {$O(T^{1-\beta/2})$}  & \small {Strongly convex cost} \\
    \small {\citet{guo2022online}}  & \small {$O(\sqrt{T})$} & \small {$O(T^{\nicefrac{3}{4}})$}  & - \\
     \small {\citet{guo2022online}}  & \small {$O(\log{T})$} & \small {$O(\sqrt{T\log{T}})$}  & \small{Strongly convex costs} \\
     %   \small {\citet{guo2022online}}  & \small {$O(\log T)$} & \small {$O(\sqrt{T \log T})$} & \small {Conv-OPT} & \small {Strongly convex cost} \\ 
  \small {\citet{yi2023distributed}}  & \small {$O(T^{\max(\beta, 1-\beta)})$} & \small {$O(T^{1-\beta/2})$} & \small{$0\leq \beta \leq 1$ is adjustable}\\
    %\small {\citet{yi2023distributed}} & \small {$O(\log(T))$} & \small {$O(\sqrt{T \log T})$} & \small {Conv-OPT} & \small {Strongly convex cost} \\
    %\citet{georgios-cautious} & Adversarial, convex & $S$ & $O(\sqrt{ST})$ & OGD & Known $S$ \\
               %\citet{guo2022online} & Adversarial, strongly-convex & $1$ & $O(\sqrt{T \log T})$ & Convex opt. each round & -do-, Known $\alpha$ \\
   %      \small {\citet{georgios-cautious}}  & &\ \textcolor{red}{($S$,$O(\sqrt{ST})$)}& \small {OGD} & \small {Known $S$} \\
  %\small {\textbf{This paper}}  & \small {$O(\sqrt{T})$} & \small {$O(T^{3/4})$} & \small {Ad-OCO} & - \\
     \small {\citet{sinha2024optimal}} & \small {$O(\sqrt{T})$} & \small {$O(\sqrt{T}\log T)$}  & -\\
     \small {\bf This paper} & \small {$O(\sqrt{T\log T})$} & \small {$O((\log T)^2)$}  & -\\
       \small {\bf This paper} & \small {$O(\log T)$} & \small {$O((\log T)^2)$}  & \small{
       Exp-concave/strongly convex costs}\\
       \bottomrule
  \end{tabular}
  \vspace{5pt}
  \caption{\small{Summary of results on COCO with arbitrary convex costs and convex constraint functions, both of which could be adversarially chosen at every round. Only the variation of the bounds with respect to $T$ is shown, treating all other parameters to be constant.}}
    \label{gen-oco-review-table}
\end{table*}
\input{prelims}

\input{prelim2}

\section{Problem formulation}\label{sec:model}
\subsection{Model and notation}
Let $\X\subset\R^d$ be a compact and convex decision set, with a finite diameter $D>0$. On round $t$, the learner plays an action $x_t\in\X$ and then observes convex, $G$-Lipschitz cost and constraint functions $f_t$ and $g_t$, chosen by an adaptive adversary. Consequently, the algorithm incurs a cost of $f_t(x_t)$ and a constraint violation of $\max\big(0, g_t(x_t)\big)$ for round $t$. We assume a known horizon $T\ge2$ and a nonempty common feasible set given by  
\begin{eqnarray} \label{feas-set}
 \mathcal{X}^\star=\{u\in\X:g_t(u)\le0\text{ for all }1 \le t\le T\}.
\end{eqnarray}

The goal is to design an online policy that minimizes the following two performance measures by competing against the best fixed commonly feasible action from the set $\mathcal{X}^*:$
\begin{equation}\label{eq:metrics}
 \Regret_T=\max_{u\in\X^\star}\sum_{t=1}^T(f_t(x_t)-f_t(u)),
 \qquad \CCV_T=\sum_{t=1}^T(g_t(x_t))_+.
\end{equation}
%For every round,
%the loss oscillation is at most $GD$ and $(g_t(x_t))_+\le GD$, by comparison with any $u\in\X^\star$.
 %For a real function $h$ on $A$, write
%$\operatorname{osc}_A h=\sup_A h-\inf_A h$.
 For technical reasons, we will find it useful to work with the $\delta$-relaxed constraints $g_t(x) \leq \delta$ where $\delta$ is a small positive parameter.  
For any fixed $\delta >0$, define the sequence of $\delta$-feasible sets to be 
\begin{equation}\label{eq:sets}
 K_0=\X,\qquad K_t=K_{t-1}\cap\{x:g_t(x)\le\delta\}.
\end{equation}
Clearly, the sets $\{K_t\}_{t=0}^T$ are nested, compact, and convex. The strict relaxation of the constraints by $\delta >0 $ is essential: even when \(\mathcal{X}^\star\) has zero volume, it guarantees that a \(T^{-1}\)-scaled neighborhood of every feasible comparator survives all eliminations, which will be useful for bounding the $\mathsf{CCV}$.  
%In the sequel, we will also make use of  a sequence of time-varying probability densities $p_t: K_t \to \mathbb{R}_+$ induced by the active Hedge algorithm which will be described later.

\section{The $\mathsf{Active~Hedge}$ Algorithm on a Continuum}\label{sec:algorithm}
For convenience, we start with the following simplification.
\begin{framed}
\paragraph{Preprocessing:} We assume without loss of generality that \(\mathcal{X}\) is full-dimensional in \(\mathbb R^d\). We rescale \(\mathcal{X}\) so that \(\operatorname{vol}_d(X)=1\), in which case the Lebesgue measure coincides with the uniform probability measure on \(\mathcal{X}\). More generally, \(d\) may be taken to be the affine dimension of \(\mathcal{X}\) by restricting the problem to the affine hull of \(\mathcal{X}\). 
 
%We assume that the standard $d$-dimensional (Lebesgue) volume of $\mathcal{X}$ is strictly positive, \emph{i.e.,} $\mathsf{vol}_d(\mathcal{X}) >0.$ In particular, we scale the decision set $\mathcal{X}$ so that $\mathsf{vol}_d(\mathcal{X})=1.$ In this case, the standard Lebesgue measure $dx$ coincides with the uniform measure on $\mathcal{X},$ denoted by $d\mu.$ With this scaling, the probability of any measurable subset under the uniform measure becomes equal to its volume.  In the general case, we can replace the Lebesgue measure $dx$ with $d\mu$ throughout the paper.
\end{framed}

We use an active version of the continuous Hedge algorithm with a non-increasing learning rate schedule $\{\eta_t\}_{t\geq 1}$ \citep{krichene2015}. 
Set the parameters $\delta=\frac{GD}{T}$ and the initial learning rate\footnote{All logarithms in this paper are taken with respect to the natural base.} $\eta_1=\frac1{GD}\min\left\{1,\sqrt{\frac{8 d \log T}{T}}\right\}$. Let $F_{t-1}$ be the cumulative cost function up to and including round $t-1$ \emph{i.e.,} $F_{t-1}(x):=\sum_{s=1}^{t-1} f_s(x).$
The density $p_t(x)$ chosen by the $\mathsf{Active~Hedge}$ policy on round $t$ and the decision $x_t$ played by it are given by 
\begin{equation}\label{eq:hedge}
 p_t(x)=\frac{\ind_{K_{t-1}}(x)e^{-\eta_tF_{t-1}(x)}}{Z_{t-1}(\eta_t)},
 \qquad x_t=\int_{K_{t-1}} x\,p_t(x)\,dx,
\end{equation}
where we have defined the normalization constant
\begin{equation}\label{eq:partition}
 Z_{t-1}(\eta)=\int_{K_{t-1}}e^{-\eta F_{t-1}(x)}\,dx.
\end{equation}
Clearly, $p_1$ is the uniform probability density over the decision set $\mathcal{X}.$
It will be shown in Lemma \ref{lem:volume} that $\mathsf{vol}_d(K_{t-1}) >0$ and hence, the normalization constant is strictly positive. 

Next, we describe how the learning rates are updated. 
Intuitively, as the feasible sets, and consequently, the set of benchmarks, becomes smaller due to the elimination of the infeasible sets, the learning algorithm should be more stable (i.e., smaller update steps) in order to minimize the movement of the successive iterates. More precisely, for any $t\geq 1$, we choose to update the learning rate as 
\begin{eqnarray} \label{eq:rate-dec}
	\eta_{t+1}=\begin{cases}
		\eta_t\quad \text{if}~~ g_t(x_t) \leq \delta \\
		\eta_t\big(1+\frac{0.1}{d\log T}\big)^{-1}\quad \text{o.w.}
	\end{cases}
\end{eqnarray}
The motivation for this particular learning rate update comes from the proof of the regret bound given in Section \ref{sec:regret}. 
\begin{algorithm}[ht]
\caption{Continuous Hedge with Elimination ($\eta_1$)}\label{alg:main}
\textbf{Initialize:} $K_0=\X$, $F_0=0$, $\delta = \frac{GD}{T},\eta_1$.\\
\textbf{For} $t=1,\ldots,T$:
\begin{enumerate}
\item Play the mean $x_t$ of the $\mathsf{Hedge}$ distribution given in Eqn.\ \eqref{eq:hedge}.
\item Observe $f_t,g_t$, set $F_t=F_{t-1}+f_t$, and update the feasible set to $K_t$ as in~\eqref{eq:sets}. 
\item Update the learning rate using Eqn.\ \eqref{eq:rate-dec}.
\end{enumerate}
\end{algorithm}
The following is our main result.
\begin{framed} 
\begin{theorem}\label{thm:main}
Assume that all cost and constraint functions are convex  and $G$-Lipschitz. Then Algorithm~\ref{alg:main}, with initial learning rate  $\eta_1 = \frac1{GD}\min\left\{1,\sqrt{\frac{8d \log T}{T}}\right\}$ achieves $O(GD\sqrt{dT\log T})$ $\Regret$ and $O(GD\,d\log^2T)$ $\CCV$ for any $d$-dimensional convex decision set with diameter at most $D$.
\end{theorem}
\end{framed}
%\footnotetext{It is also valid and computationally advantageous to update the feasible sets $K_t$ lazily only when a $\delta$-violation occurs. See the discussion in Section \ref{implementation}.}

\paragraph{Tightness of the bounds:}
\citet[Theorem 4.17]{sarkar2026geometric} established that any online algorithm for COCO must have a regret of at least $\Omega(\sqrt{T})$ and $\CCV$ of $\Omega\bigg( \frac{(\log T)^{\frac{d-1}{d+1}}}{\ln \ln T}\bigg).$ Thus, the regret bound is optimal up to a \(\sqrt{\log T}\) factor, while the CCV bound is optimal up to a factor of $O((\log T)^{1+\frac{2}{d+1}})$.
%Hence, it follows that our bounds are tight up to a poly-log factor.   
\paragraph{Observation: Lazy Updates of Feasible Sets}
Suppose the action $x_t$ remains $\delta$-feasible, i.e.,
$g_t(x_t) \leq \delta$. We can then alternatively skip updating the feasible set on round $t$ and retain
$K_t=K_{t-1}$. Equivalently, when no $\delta$-violation occurs at round $t,$ we
can replace the original constraint function $g_t$ with a fictitious constraint $\hat{g}_t=0$ without affecting the algorithm and its analysis. Thus, we 
can safely perform the update in~\eqref{eq:sets} only when $x_t$ is
$\delta$-infeasible (\emph{i.e.}, $g_t(x_t)>\delta$). From the proof of the $\CCV$ bound,
Eq.~\eqref{eq:count} implies that the number of $\delta$-violations, and
hence the total number of feasible-set updates with lazy update is at most
$O((\log T)^2)$. This helps reduce the computational complexity of the algorithm. 
\section{Proof of Theorem \ref{thm:main}}\label{sec:bounds}

\subsection{Two Geometric Results}\label{sec:observations}
The proof of Theorem \ref{thm:main} builds upon the following two geometric results. Let $\mu$ be the uniform probability measure on $\mathcal{X}.$ 
\begin{lemma}[A surviving comparator set]\label{lem:volume}
For any feasible action $u\in\X^\star$, given by Eqn.\ \eqref{feas-set}, construct a scaled and translated copy of $\mathcal{X}$ as $C_u=u+(\X-u)/T.$ The set $C_u$ has the following properties:
\begin{eqnarray}\label{eq:comparator}
 &&(1)~C_u\subseteq K_t\quad, 0\le t\le T,\\
 &&(2)~\mu(C_u)=T^{-d}, \\
 &&(3)~F_T(x)\le F_T(u)+GD, \forall x\in C_u.
\end{eqnarray}
\end{lemma}
\begin{proof}
(1) Convexity gives $C_u\subseteq\X$. Any element $y \in C_u$ can be expressed as $y = u+ \frac{x-u}{T}$ for some $x \in \mathcal{X}.$ Using the $G$-Lipschitzness of the constraint functions, we have for any $t\geq 1:$
\begin{eqnarray} \label{g_t-feas}
	g_t(y)\le g_t(u)+ G||y-u||\stackrel{(a)}{\le} 0 + \frac{G}{T}||x-u|| \stackrel{(b)}{\leq} \frac{GD}{T}=\delta. 
\end{eqnarray}
where inequality (a) follows from the feasibility of the action $u$ and (b) follows from the fact that the decision set has a finite diameter of $D.$
Inequality \eqref{g_t-feas} shows that $y \in K_t,$ and hence, $C_u \subseteq K_t.$

(2) Since $C_u$ is obtained by first scaling $\mathcal{X}$ by $T^{-1}$ and then translating it, we have
\begin{eqnarray*}
	\mu(C_u) = \mu(\frac{1}{T} \mathcal{X}) = T^{-d}\mu(\mathcal{X})= T^{-d}.
\end{eqnarray*} 

 (3) Note that the cumulative cost function $F_T$ is $GDT$-Lipschitz. Hence, proceeding similar to Eqn.\ \eqref{g_t-feas}, we have for any $x \in C_u:$
\[F_T(x)-F_T(u)\le TGD/T=GD.\]

\end{proof}
Part (1) and (2) of Lemma \ref{lem:volume} imply that  $\mu(K_t)>0$, even if $\X^\star$ is a singleton. Continuity and compactness imply $0<Z_t(\eta)<\infty$ for every $\eta>0$. Inductively, $0<\eta_{t+1}\le\eta_t$, and $\eta_tGD \leq \eta_1 GD \le1, \forall t\geq 1$. 

We now appeal to Lemma \ref{lem:halfspace} to lower bound the probability of the eliminated sets on each constraint violation. 
For technical reasons, we will need to lower bound the probability under a one-step updated interim density $p_{t+\nicefrac{1}{2}}$ as described next. 
%Imagine that, after taking the action $x_t,$ we first observe the next cost function $f_t,$ before observing the constraint function $g_t.$ 
After \(x_t\) is chosen, condition on the realized cost function \(f_t\) and define a new interim density $p_{t+\nicefrac{1}{2}}$ by incorporating \(f_t\) into the current Hedge distribution, before applying the constraint update.
In other words, the interim probability density $p_{t+\nicefrac{1}{2}}$ is defined by replacing $F_{t-1}$ in \eqref{eq:hedge} with the updated cumulative cost $F_t = F_{t-1}+f_t.$ This yields the following expression for the interim density function: 
\begin{eqnarray*}
	p_{t+\nicefrac{1}{2}}(x):=\frac{\ind_{K_{t-1}}(x)e^{-\eta_t F_{t}(x)}}{\int_{K_{t-1}}e^{-\eta_tF_t(x)}\,dx}.
\end{eqnarray*}

After the constraint function $g_t$ is revealed by the adversary, let $q_t$ be the probability that a point in $K_{t-1}$ belongs to the updated feasible set $K_t$ under the interim density $p_{t+\nicefrac{1}{2}}.$ We have \begin{equation}\label{eq:rate}
 q_t=\frac{\int_{K_{t}}e^{-\eta_tF_t(x)}\,dx}{\int_{K_{t-1}}e^{-\eta_tF_t(x)}\,dx} = \frac{Z_t(\eta_t)}{\int_{K_{t-1}}e^{-\eta_tF_t(x)}\,dx}.
\end{equation}

%Next, recall that the mean $x_t$ lies in the convex set $K_{t-1}$. An application of Grunbaum's lemma (Lemma \ref{lem:halfspace}) yields the following. 
The following lemma lower bounds the probability of the eliminated set under both $p_t$ and $p_{t+\nicefrac{1}{2}}.$ See Figure \ref{fig:gb} for an illustration.

\begin{figure*}
\centering
	\includegraphics{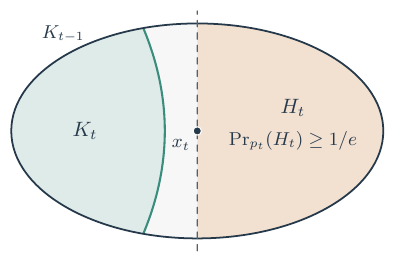}
	\caption{\small{Gr\"unbaum's inequality showing that a large volume of constant hedge probability is eliminated at each round of constraint violation.}}
	\label{fig:gb}
\end{figure*}

\begin{lemma}[A violated mean removes constant $\mathsf{Hedge}$ probability]\label{lem:grunbaum}
If $g_t(x_t)>\delta$, then
\begin{equation}\label{eq:mass}
 \int_{K_{t-1}\setminus K_t}p_t(x)\,dx\ge e^{-1},
 \qquad \int_{K_{t-1}\setminus K_t} p_{t+\nicefrac{1}{2}}(x) dx := 1-q_t \ge e^{-2}.
\end{equation}
\end{lemma}
\begin{proof}
The density $p_t,$ defined in Eqn.\ \eqref{eq:hedge}, is log-concave because the set $K_{t-1}$ and the cumulative cost function $F_{t-1}$ are convex. Furthermore,
due to the violation, its mean $x_t\notin K_t.$ Hence, the convex set consequence~\eqref{eq:convexcut} of Lemma~\ref{lem:halfspace}, with $K=K_{t-1}$ and $C=K_t$, gives the first inequality.

Furthermore,
%The loss update tilts $p_t$ by $e^{-\eta_tf_t}$, giving
\begin{eqnarray*}
 1-q_t&=&\frac{\int_{K_{t-1}\setminus K_t}e^{-\eta_tF_t(x)}\,dx}
 {\int_{K_{t-1}}e^{-\eta_tF_t(x)}\,dx} \\
 &=& \frac{\int_{K_{t-1}\setminus K_t}p_t(x)e^{-\eta_tf_t(x)}\,dx}
 {\int_{K_{t-1}}p_t(x)e^{-\eta_t f_t(x)}\,dx} \\
 &\ge&  \frac{e^{-\eta_t \max_{y \in K_{t-1}} f_t(y)}\int_{K_{t-1}\setminus K_t}p_t(x)\,dx}{e^{-\eta_t \min_{y \in K_{t-1}} f_t(y)}\int_{K_{t-1}}p_t(x)\,dx} \\
 &=& e^{-\eta_t (\max_{y \in K_{t-1}} f_t(y)- \min_{y \in K_{t-1}} f_t(y))} \int_{K_{t-1}\setminus K_t}p_t(x)\,dx\\
 &\stackrel{(a)}{\ge}&  e^{-\eta_t GD}e^{-1} \\
 &\stackrel{(b)}{\ge}& e^{-2},
\end{eqnarray*}
where (a) follows from the Lipschitz continuity of $f_t$ and the first part of the Lemma and (b) follows from the fact that $\eta_t GD \leq \eta_1 GD \le 1, \forall t \geq 1.$
%Finally, $-\log q_t\ge1-q_t \geq e^{-2}$.
\end{proof}

\paragraph{The role of $\delta$.}
Lemma~\ref{lem:grunbaum}'s proof applies to any threshold $\delta\ge0$. The constants do not depend on $\delta$ as long as $\delta \geq 0$.
%If $q_t=0$, the mass bounds still hold, with $\log(1/q_t)=+\infty$, but the next density cannot be normalized.
%For example, take $\X=[0,1]$, $f_1=0$, and $g_1(x)=x$. At $\delta=0$, the initial uniform density has mean $x_1=1/2$, while $K_1=\{0\}$ and $q_1=0$.
%Thus the lemma remains valid although the algorithm cannot continue with the same reference measure.
%Even positive surviving volume alone gives no bound of the form $\mu(K_t)\ge T^{-d}$.
%The choice $\delta=GD/T$ supplies this bound through Lemma~\ref{lem:volume}; it also limits the total contribution of rounds with $g_t(x_t)\le\delta$ to $\delta T=GD$.
%
%\begin{figure}[t]
% \centering
% \includegraphics[width=\linewidth]{figures/continuous_geometry.pdf}
% \caption{The two geometric observations. Left: a relaxed constraint excludes the mean of a log-concave Hedge density. A halfspace through the mean has probability at least $1/e$, all of which is deleted. Right: the set $C_u=u+(\X-u)/T$ lies within distance $D/T$ of a common feasible point $u$, survives all relaxed constraints, and has relative volume $T^{-d}$. The illustrated value $T=4$ makes this copy visible.}
% \label{fig:geometry}
%\end{figure}
%

\subsection{A Potential Function for Bounding the $\CCV$}\label{sec:potentialbackground}
%In the Lyapunov analysis of \citet{sinha2024optimal}, the accumulated positive violation determines the constraint weight in the surrogate loss. Bounding the drift of a potential then controls CCV.
The discrete version of the $\mathsf{Active~Hedge}$ algorithm, proposed by \citet{sarkar2026geometric}, uses a simple potential function that is equal to the number of \emph{surviving experts} to bound the $\CCV.$ The argument is that at least one expert is eliminated at each violation and hence, the cumulative $\CCV$ is upper bounded by the total number of experts. We propose a natural generalization of this idea to the continuous setting by considering the volume of the \emph{surviving sets} in the uniform measure. Unfortunately, this is not enough to bound the $\CCV$ as Gr\"unbaum's inequality in Lemma \ref{lem:halfspace} guarantees a decrease in the probability of the surviving sets only the time-varying measure induced by the Hedge algorithm - \emph{not} in the uniform measure. 

To get around this difficulty, we also include an extra additive term involving the time-varying learning rates $\eta_t$ in the potential function. %
%Suppose a surviving set has relative volume at least $e^{-H}$. Its remaining logarithmic volume is
%$\log(\mu(K_{t-1})/e^{-H})$.
%For $N$ finite experts with uniform reference measure and $H=\log N$, this expression is exactly $\log|S_{t-1}|$.
%This explains the volume term as a continuous counterpart of a count of surviving experts.
%However, a constant amount of Hedge probability can occupy a very small volume when the density is concentrated. A volume term alone therefore does not account for every deletion supplied by Gr\"unbaum's inequality.
Since the variation of the cumulative cost functions 
$F_t$ is at most $GDT$, the density \eqref{eq:hedge} on its support has maximum-to-minimum ratio of at most $e^{GDT\eta_t}$.
When \(GDT\eta_t\geq1\), we charge the progress from elimination to the decrease in the learning rate.  Once $GDT\eta_t\le1$, the density is within a factor $e$ of uniform, and constant deleted probability implies constant deleted volume.
These observations motivate the following potential function:
\begin{equation}\label{eq:potential}
 \Psi_t=\underbrace{\log\frac{\mu(K_{t-1})}{e^{-H}}}_{\textrm{geometric progress}}
       +\underbrace{H\log(1+GDT\eta_t)}_{\textrm{learning-rate progress}},
\end{equation}
with $H \equiv d\log T.$
%The addition of $1$ keeps the second term nonnegative as the rate approaches zero. Its coefficient $H$ matches the normalization of the rate update below.
 Section~\ref{sec:ccv} shows that the two terms together decrease by a constant on every round with $g_t(x_t)>\delta$. Thus the proof uses one potential \eqref{eq:potential} throughout, without a restart or a uniform-volume deletion assumption.

\subsection{Proof of the CCV Bound}\label{sec:ccv}

Let $M$ be the number of rounds with $g_t(x_t)>\delta$. We call such a round to be ``counted". For bounding the $\CCV,$ we will use the potential $\Psi_t.$
By Lemma~\ref{lem:volume}, the potential $\Psi_t$ is non-negative as $\mu(K_t) \geq T^{-d}=e^{-H}$. Both terms are non-increasing, and
\begin{equation}\label{eq:initialpotential}
 \Psi_1\le H[1+\log(1+ GDT \eta_1)].
\end{equation}
We show that every counted round decreases the potential $\Psi_t$ by at least a constant. Consider two cases:

\textbf{\underline{Case-I}} $s:=GDT\eta_t\ge1:$ In this case the second term decreases by
\begin{eqnarray*}
 H\log\frac{1+s}{1+s/(1+0.1H^{-1})}
 &\stackrel{(a)}{\ge}& H\log\frac{2}{1+1/(1+0.1H^{-1})}\\
 &=&H\log \big(1+ \frac{0.1H^{-1}}{2+0.1H^{-1}}\big) \\
 &\stackrel{(b)}{\geq}& H \log \big(1 +  \frac{0.1H^{-1}}{2+0.1/\log 2}\big) \stackrel{(c)}{\geq} 0.02.
\end{eqnarray*}
where (a) follows from the non-decreasing property of the function $x \mapsto \frac{1+x}{1+x/(1+0.1H^{-1})}$, (b) follows from the fact that $H\equiv d\log T\ge\log2,$ and (c) follows from the inequality $\log(1+x) \ge \frac{x}{2},$ valid for $0 \leq x \leq 2.5$.

\textbf{\underline{Case-II}} $GDT\eta_t\le1:$ In this case, we have 
\[\eta_t (\max_{x \in K_{t-1}} F_{t-1}(x)- \min_{y \in K_{t-1}} F_{t-1}(y))\le\eta_t(t-1)GD\le1.\]
Hence, the maximum and minimum values of the Hedge density \eqref{eq:hedge} on $K_{t-1}$ therefore differ by at most a factor $e$. For any measurable $E\subseteq K_{t-1}$, this implies
\[
 \int_Ep_t(x)\,dx\le  \frac{\int_{E}\max_{x \in K_{t-1}} e^{-\eta_t F_{t-1}(x)} dx}{\int_{K_{t-1}}\min_{x \in K_{t-1}} e^{-\eta_t F_{t-1}(x)} dx} \leq  e\,\frac{\mu(E)}{\mu(K_{t-1})}.
\]
Taking $E=K_{t-1}\setminus K_t$ and applying Lemma~\ref{lem:grunbaum} gives a removed volume fraction at least $e^{-2}$. Thus, once the density is approximately uniform, a constant amount of Hedge mass can only be removed by deleting a constant fraction of the underlying volume.

Hence the first term of~\eqref{eq:potential} decreases by at least
\[
 \log\frac{\mu(K_{t-1})}{\mu(K_t)}\ge\log(1+e^{-2})\ge 0.1.
\]
Taking the above two cases together, we conclude that for any counted rounds, the potential decreases by at least $0.02.$ Since the potential is non-increasing and lower bounded by zero, the number of counted rounds must be upper bounded as  
%These two cases analyze the same algorithm and the same potential; there is no restart. 
telescoping gives
\begin{equation}\label{eq:count}
 M\le\frac{\Psi_1}{0.02}\le 50H[1+\log(1+GDT\eta_1)].
\end{equation}
Every counted round contributes at most $GD$ to $\CCV$. Every other round contributes at most $\delta$, so
\begin{equation}\label{eq:ccv}
 \CCV_T\le\delta T+GD\,M \leq GD+ 50 GDd\log T(1+\log(1+GDT\eta_1)).
\end{equation}
Substituting the value of $\eta_1$ proves the $\CCV$ bound. 
%In particular, the argument does not require every violation to remove a constant fraction of uniform volume: Gr\"unbaum controls the current Hedge probability.

\subsection{Proof of the Regret bound}\label{sec:regret}
We extend the proof of the regret bound of the Continuous Hedge algorithm
\citep{krichene2015,sarkar2026geometric}, while additionally accounting for the probability of the eliminated sets. Define the potential function
\begin{equation}\label{eq:freeenergy}
 \Phi_t(\eta):=-\frac1\eta\log Z_t(\eta).
\end{equation}
where $Z_t(\eta)$ is the normalization constant of the continuous $\mathsf{Hedge}$ algorithm \eqref{eq:partition} at the end of round $t$ for the learning rate $\eta >0$. 

We first claim that $\eta \mapsto \Phi_t(\eta)$ is a non-increasing function for any $t\geq 1.$ To see this, consider any $0<\eta'\le\eta$. Applying Jensen's inequality to the concave function $v\mapsto v^{\eta'/\eta}$ under the uniform probability measure yields
\[
 Z_t(\eta')=\int_\X\bigl(\ind_{K_t}e^{-\eta F_t}\bigr)^{\eta'/\eta}\,dx
 \le Z_t(\eta)^{\eta'/\eta}.
\]
%Hence, for any $t\geq 1,$ the potential function $\Phi_t(\cdot)$ is non-increasing in its argument. 
In the above application of Jensen's inequality, we are using the equality of the uniform and Lebesgue measures over $\mathcal{X}.$
The claim follows upon taking natural logarithms of both sides. 

Next, from Eqn.\ \eqref{eq:rate}, observe that the successive normalization constants are related to the probability $q_t$ of the surviving set as: 
\begin{eqnarray*}
	Z_t(\eta_t) = q_t \int_{K_{t-1}} e^{-\eta_t F_{t-1}(x)} e^{-\eta_t f_t(x)} dx \stackrel{(a)}{=} q_t Z_{t-1}(\eta_t)\int_{K_{t-1}} p_t(x) e^{-\eta_t f_t(x)} dx,
\end{eqnarray*} 
where in (a), we have substituted for the $\mathsf{Hedge}$ density $p_t$ from Eqn.\ \eqref{eq:hedge}. Taking logarithm of both sides, the successive potential functions are related as:
\begin{eqnarray} \label{potential-lemma}
	\Phi_t(\eta_t) = \Phi_{t-1}(\eta_t) - \frac{\log q_t}{\eta_t} - \frac1{\eta_t}\log\int_\X p_t e^{-\eta_tf_t}\,dx.
\end{eqnarray}

%Convexity of $f_t$, Hoeffding's lemma with loss oscillation at most $GD$, and the definition of $q_t$ give
Equipped with the above relation, we can upper bound the cost incurred on round $t$ as 
\begin{align} \label{loss-bd}
 f_t(x_t)&\stackrel{(a)}{\le}\int_\X f_t(x) p_t(x)\,dx \nonumber\\
 &\stackrel{(b)}{\le}-\frac1{\eta_t}\log\int_\X p_t e^{-\eta_tf_t}\,dx
          +\frac{G^2D^2\eta_t}{8}\nonumber\\
 &\stackrel{(c)}{=}\Phi_t(\eta_t)-\Phi_{t-1}(\eta_t)
          +\frac{\log q_t}{\eta_t}+\frac{G^2D^2\eta_t}{8} \nonumber\\
 &\stackrel{(d)}{\le} \Phi_t(\eta_t)-\Phi_{t-1}(\eta_{t-1})
          +\frac{\log q_t}{\eta_t}+\frac{G^2D^2\eta_t}{8}        
\end{align}
where (a) follows from the convexity of the cost function $f_t$ with $x_t$ being the mean of $p_t,$ (b) follows from Hoeffding's Lemma \citep[Lemma 2.2]{boucheron2013concentration} used with the fact that the variation of $f_t$ on $\mathcal{X}$ is bounded by $GD$, (c) follows from Eqn.\ \eqref{potential-lemma} and (d) follows from the non-increasing property of the potential function $\Phi_{t-1}(\cdot)$ as $\eta_t \leq \eta_{t-1}$. 
Summing up inequality \eqref{loss-bd} from $t=1$ to $t=T$ and using the fact that $\Phi_0 =0,$ the cumulative cost incurred by the online policy can be upper-bounded as 
\begin{eqnarray} \label{cumulative-loss-bd}
	\sum_{t=1}^T f_t(x_t) \leq \Phi_T(\eta_T) + \frac{G^2D^2}{8}\sum_{t=1}^T\eta_t+
 \sum_{t=1}^T\frac{\log q_t}{\eta_t}.
\end{eqnarray}
It now remains to upper bound the final potential $\Phi_T(\eta_T).$
For every $u\in\X^\star$, Lemma~\ref{lem:volume} yields
\[
 Z_T(\eta)\ge\int_{C_u}e^{-\eta F_T(x)}\,dx
 \ge e^{-H-\eta[F_T(u)+GD]},
\]
and hence $\Phi_T(\eta)\le F_T(u)+GD+H/\eta$. Using this bound in \eqref{cumulative-loss-bd} yields the following regret bound:
\begin{eqnarray} \label{reg-bd1}
 \Regret_T\le\frac H{\eta_{T+1}}+GD+
 \frac{G^2D^2}{8}\sum_{t=1}^T\eta_t+
 \sum_{t=1}^T\frac{\log q_t}{\eta_t}.
\end{eqnarray}
Next, we claim that for any round $t\geq 1,$ we have 
\begin{equation}\label{eq:cancellation}
 H\left(\frac1{\eta_{t}}-\frac1{\eta_{t+1}}\right)
 \geq \frac{\log(q_t)}{\eta_t}.
\end{equation}
The above follows from the learning rate update \eqref{eq:rate-dec} by considering two separate cases. 

\textbf{Case I: $g_t(x_t) \leq \delta:$} In this case, \eqref{eq:cancellation} holds trivially $\eta_t = \eta_{t+1}$ and $q_t \leq 1.$

\textbf{Case II: $g_t(x_t) > \delta:$} In this case from the update rule \eqref{eq:rate-dec}, we have 
\begin{eqnarray*}
	\frac{1}{\eta_t} - \frac{1}{\eta_{t+1}} = -0.1\frac{H^{-1}}{\eta_t} \stackrel{(a)}{\geq} \log(q_t) \frac{H^{-1}}{\eta_t},
\end{eqnarray*}
where (a) follows from Lemma \ref{lem:grunbaum}. 
Substituting the bound \eqref{eq:cancellation} into the regret bound \eqref{reg-bd1} and using the fact that $\eta_t\le\eta_1$, we have 
\begin{equation}\label{eq:regret}
 \Regret_T\le\frac H{\eta_1}+\frac{G^2D^2T\eta_1}{8}+GD.
\end{equation}
For $T\ge8H$, substituting the value of $\eta_1$ gives
$GD\sqrt{TH/2}+GD\le3GD\sqrt{TH}$.
For $T<8H$, use $\Regret_T\le GDT<\sqrt8\,GD\sqrt{TH}$.
This proves the regret bound in Theorem~\ref{thm:main}. 

%\input{implementation}
\input{extension}

\input{conclusion}

\begingroup
%\small
%\setlength{\bibsep}{3pt}
%\bibliographystyle{plainnat}
\bibliography{references, OCO.bib, refs.bib}
\endgroup
\end{document}

%% file: abstract.tex
\begin{abstract}
%We study constrained online convex optimization with adversarial convex losses and constraints (\COCO). At each round $t \in [T]$, a learner selects a point $x_t$ from a $d$-dimensional convex decision set $\mathcal{X}$. Subsequently, an adversary chooses a convex cost function $f_t$ and a convex constraint function $g_t$, resulting in a cost of $f_t(x_t)$ and constraint violation of $\max(0, g_t(x_t))$ for the learner on round $t.$ The goal for the learner is to select actions to simultaneously minimize the regret corresponding to the cost functions and the cumulative constraint violation ($\CCV$), corresponding to the constraint functions. The previously best-known algorithms for \COCO~achieve $O(\sqrt{T})$ regret and $O(\sqrt{T})$ $\CCV$. In this paper, we propose an online policy that achieves $O(\sqrt{T \log T})$ regret and $O(\log^2(T))$ $\CCV$ - improving the $\CCV$ from polynomial to polylogarithmic. Both the bounds are optimal up to a logarithmic factor.  Our main idea is to run the continuous Hedge algorithm on the currently feasible set with non-increasing learning rates which are updated according to the reduction of the feasible sets in some measure. The proof of the regret bound follows from a generalization of the standard arguments while taking into account the probability of the eliminated sets. A potential function involving surviving volume and the learning rate concludes the $\CCV$ bound after appealing to Gr\"unbaum's inequality. 

We study constrained online convex optimization with adversarial convex costs and constraints (\COCO). At each round \(t\in[T]\), a learner selects \(x_t\) from a \(d\)-dimensional convex decision set \(\mathcal X\), after which an adaptive adversary reveals a convex cost function \(f_t\) and constraint function \(g_t\). Consequently, the learner incurs a cost of \(f_t(x_t)\) and a constraint violation of \(\max\{0,g_t(x_t)\}\), and aims to simultaneously minimize regret and cumulative constraint violation ($\CCV$) over the entire horizon. Existing algorithms achieve \(O(\sqrt{T})\) regret and \(\widetilde O(\sqrt{T})\) $\CCV$. We show that an online policy can achieve \(O(\sqrt{T\log T})\) regret and \(O(\log^2 T)\) $\CCV$, reducing the $\CCV$ from polynomial to polylogarithmic while retaining near-optimal regret. When all cost functions are either $\alpha$-exp-concave or $\beta$-strongly convex, the algorithm achieves $O(\log T)$ regret with the same $\CCV$ bound as above. Our approach combines continuous Hedge with elimination on shrinking feasible sets. The key observation is that whenever the mean of the Hedge distribution violates a constraint, Gr\"unbaum's inequality guarantees that a constant fraction of the Hedge probability mass is eliminated. We use an adaptive learning-rate schedule and a potential function coupling the surviving volume with the learning rate to convert this probability-mass reduction into a bound of \(O(\log^2 T)\) on the $\CCV$.

\end{abstract}

%% file: prelims.tex
\section{Precursor: $\mathsf{Active~Hedge~with~Elimination}$} \label{precursor}

The algorithm presented in this paper is inspired by the $\mathsf{Active~Hedge~with~Elimination}$
algorithm \citep[Section~5]{sarkar2026geometric}, which addresses a discrete counterpart of the $\mathsf{COCO}$ problem called
$\mathsf{Constrained~Experts}$. Here, the decision set $\mathcal{X}$ is
the $(N-1)$-dimensional probability simplex $\Delta_{N-1}$, where $N$
is the number of experts. The learner competes against the experts
that satisfy every constraint in hindsight:
\[
\mathcal{X}^*_{\textsf{Experts}}
=
\bigcap_{t=1}^T \{i\in[N]:g_t(i)\leq 0\}.
\]
This benchmark restricts comparison to individual feasible experts,
whereas \eqref{feas-set} allows comparison against all feasible convex
combinations of experts.

The algorithm maintains exponential weights over the experts that
satisfy all previously observed constraints and permanently eliminates
any expert that violates the current constraint. Specifically, it
maintains a set of surviving experts $S_t$ and plays the probability
vector
\begin{equation}\label{eq:finitehedge}
p_t(i)
=
\frac{\ind\{i\in S_{t-1}\}e^{-\eta\sum_{s<t}f_s(i)}}
{\sum_{j\in S_{t-1}}e^{-\eta\sum_{s<t}f_s(j)}},
\qquad
S_t=\{i\in S_{t-1}:g_t(i)\leq 0\},
\quad S_0=[N],
\end{equation}
where $\eta>0$ is a fixed learning rate. Its cost and constraint
violation on round $t$ are $\sum_i p_t(i)f_t(i)$ and
$\max\{0,\sum_i p_t(i)g_t(i)\}$, respectively.
For cost and constraint functions taking values in $[-1,1]$,
\citet[Theorem~5.2]{sarkar2026geometric} establish regret at most
$(\log N)/\eta+\eta T/2$ against the best common feasible expert and
$\CCV$ at most $N$. The latter bound follows from a simple potential-based
argument: every round with positive violation permanently eliminates
at least one expert. The number of surviving experts therefore
decreases on each such round, and each round contributes at most one
unit of violation.

A natural extension of the $\mathsf{Active~Hedge}$ algorithm to general $\mathsf{COCO}$ is to discretize the
decision set using a uniform $d$-dimensional grid and treat the discrete grid points
as experts. At round $t,$ the online algorithm plays the mean vector induced by the $\mathsf{Active~Hedge}$ policy. By the Lipschitz continuity of the cost functions, a grid
spacing of order $T^{-1/2}$ keeps the additional regret due to
discretization within $O(\sqrt{T})$. However, such a grid contains
$N=\Theta(T^{d/2})$ experts, making the $\CCV$ guarantee of \citet{sarkar2026geometric}
uninformative in higher dimensions. The difficulty with this approach
is that it does not exploit the convexity of the decision set or the
cost and constraint functions. In particular, as we will see later, counting only one
elimination per constraint violation can substantially underestimate the
progress made by the algorithm. Instead, we must look at the \emph{mass} of the eliminated sets in an appropriate measure.  

%Moreover, although counting
%eliminated experts gives a finite bound in the discrete problem,
%it has no useful counterpart on a continuum.

We address this difficulty by identifying a potential function that
quantifies the algorithm's progress on a continuous decision set and controls its $\CCV$.
Gr\"unbaum's inequality for log-concave measures, reviewed below, shows that whenever
the mean of the current Hedge distribution violates a constraint,
the eliminated region contains a constant fraction of the
\emph{Hedge probability mass}. We also decrease the learning rate as the
feasible sets shrink, with the decrease determined by the probability
removed. A potential combining the surviving volume and the learning
rate then bounds the number of violating rounds.

%% file: prelim2.tex
\subsection{Log-concavity and Gr\"unbaum's inequality}\label{sec:grunbaum}
A nonnegative function $\rho: \mathbb{R}^d \mapsto \mathbb{R}_+$ is called \emph{log-concave} if for any $x, y \in \mathbb{R}^d,$ we have 
\[\rho(\lambda x+(1-\lambda)y)\ge\rho(x)^\lambda\rho(y)^{1-\lambda}, ~\textrm{for} ~~0\le \lambda \le 1.\]
In particular, normalizing the function $e^{-V}\ind_K : \mathbb{R}^d \mapsto \mathbb{R}_+$ gives a log-concave density when $V$ and $K$ are convex and the integral is finite and strictly positive. We use the following functional form of Gr\"unbaum's inequality.

\begin{prelemma}[Gr\"unbaum; {\citealp[Lemma~5.12]{lovasz2007geometry}}]\label{lem:halfspace}
Let $\rho$ be a log-concave probability density on $\R^d$ with finite first moment and mean $m=\int x\rho(x)\,dx$. For every nonzero $a\in\R^d$,
\begin{equation}\label{eq:halfspace}
 \int_{\{x:\langle a,x-m\rangle\ge0\}}\rho(x)\,dx\ge\frac1e.
\end{equation}
%The statement also holds in the intrinsic affine coordinates of the support.
\end{prelemma}

In particular, if $\rho$ is supported on a compact convex set $K$, and a compact convex set $C\subseteq K$ excludes $m$, then Lemma \ref{lem:halfspace} implies that 
\begin{equation}\label{eq:convexcut}
 \int_{K\setminus C}\rho(x)\,dx\ge\frac1e.
\end{equation}
Indeed, the claim is immediate if $C$ is empty. Otherwise strict separation gives $a\ne0$ with
$\sup_{x\in C}\langle a,x\rangle<\langle a,m\rangle$.
The halfspace in~\eqref{eq:halfspace} is disjoint from $C$, proving~\eqref{eq:convexcut}.
It is to be noted that since the mean vector $m$ is computed with respect to the density $\rho,$ the bound \eqref{eq:convexcut} concerns probability under the density $\rho$ only. It does not, in general, imply the same lower bound for the uniform volume.

%% file: extension.tex
%```latex
\section{Extension to Exp-Concave and Strongly Convex Cost Functions} \label{ext}

When Algorithm \ref{alg:main} is applied to the special cases where all cost functions are either $\alpha$-exp-concave or $\beta$-strongly convex, we obtain logarithmic $\Regret$ and polylogarithmic $\CCV$ bounds for appropriate choices of the initial learning rate $\eta_1$ as given in Theorem \ref{thm:main2} below. In both cases, the $\CCV$ bound \eqref{eq:ccv} remains unchanged, yielding an $O((\log T)^2)$ bound upon substituting the corresponding values of $\eta_1$. We therefore focus on establishing the improved regret bounds.

The improved logarithmic regret bound follows from a simplification of the proof of the convex case in Section \ref{sec:regret}, using the fact that, in both cases, the following stronger inequality holds:
\begin{eqnarray} \label{log-MGF}
	f_t(x_t) \le-\frac1{\eta_t}\log\int_\X p_t e^{-\eta_tf_t}\,dx.
\end{eqnarray}
Compared with the general convex case, Eqn.\ \eqref{log-MGF} strengthens inequality (b) in Eqn.\ \eqref{loss-bd} by eliminating the second term, $\nicefrac{G^2D^2 \eta_t}{8}$. Assuming \eqref{log-MGF} holds and following the same steps as in Section \ref{sec:regret}, we obtain the regret bound
\begin{eqnarray*}
	\Regret_T \leq \frac{H}{\eta_1} + GD,
\end{eqnarray*}
which improves upon the bound in \eqref{eq:regret} by removing the aforementioned middle term. Substituting the appropriate value of $\eta_1$ yields the final regret bound in each case. We now establish \eqref{log-MGF} for $\alpha$-exp-concave and $\beta$-strongly convex cost functions for appropriate choices of $\eta_1.$

\paragraph{$\alpha$-exp-concave costs:}
Recall that a function $g(x)$ is called $\alpha$-exp-concave, for $\alpha>0$, if the map $x \mapsto e^{-\alpha g(x)}$ is concave \citep{cesa2006prediction}. If all cost functions are $\alpha$-exp-concave, we set the initial learning rate to
\begin{eqnarray*}
	\eta_1 = \min \big(\alpha, \frac{1}{GD}\big).
\end{eqnarray*}
Since $\eta_t \leq \eta_1 \leq \alpha$, the cost function $f_t$ is $\eta_t$-exp-concave. Hence, \eqref{log-MGF} follows from Jensen's inequality for any distribution $p_t$, provided that $x_t$ is the mean of $p_t$.

\paragraph{$\beta$-strongly-convex costs:}
It is well-known that a $\beta$-strongly convex function with gradients uniformly bounded by $G$ is $\frac{\beta}{G^2}$-exp-concave \citep{hazan2007logarithmic, hazan2022introduction}. Hence, this case reduces to the preceding one. We summarize the resulting bounds in the following theorem.

\begin{framed}
\begin{theorem}\label{thm:main2}
\begin{enumerate}
	\item Assume that all cost functions are $\alpha$-exp-concave and $G$-Lipschitz, and that all constraint functions are convex and $G$-Lipschitz. Then Algorithm~\ref{alg:main}, with initial learning rate
	$
	\eta_1 = \min \big(\alpha, \frac{1}{GD}\big),
	$
	achieves $O((GD+\alpha^{-1})d\log T)$ $\Regret$ and $O(GD\,d\log^2T)$ $\CCV$ for any $d$-dimensional convex decision set of diameter at most $D$.

	\item Assume that all cost functions are $\beta$-strongly convex and $G$-Lipschitz, and that all constraint functions are convex and $G$-Lipschitz. Then, Algorithm~\ref{alg:main}, with initial learning rate
	$
	\eta_1 = \min \big(\frac{\beta}{G^2}, \frac{1}{GD}\big),
	$
	achieves $O((G^2/\beta+GD)d\log T)$ $\Regret$ and $O(GD\,d\log^2T)$ $\CCV$ for any $d$-dimensional convex decision set of diameter at most $D$.
\end{enumerate}
\end{theorem}
\end{framed}

%% file: conclusion.tex
\section{Conclusion}
We show that continuous Hedge with elimination achieves $\tilde{O}(\sqrt{T})$ regret and $O((\log T)^2)$ cumulative constraint violation on a general convex decision set. With an appropriately chosen initial learning rates, the same algorithm achieves $O(\log T)$ regret and $O((\log T)^2)$ $\CCV$ when the cost functions are either $\alpha$-exp-concave or $\beta$-strongly convex. The small constraint relaxation preserves a positive-volume comparator set, while Gr\"unbaum's inequality directly turns a violated mean into deleted Hedge probability. The adaptive learning-rate schedule converts this probability-mass reduction into progress in both the regret analysis and a unified CCV potential. The CCV analysis is based on a potential combining the volume of the surviving feasible set with the adaptive learning rate. Our algorithm involves exact integration over the decision set. Developing an efficient implementation based on log-concave sampling \citep{lovasz2004hit} is an interesting direction for future work. %Our algorithm involves exact integration over the decision set. Efficient implementation with log-concave sampling technique is reserved as a future work. 

\section{Acknowledgment}
This work is supported by the Department of Atomic Energy, Government of
India, under Project No.~RTI4014 and a Google India faculty research award. %The second author thanks Prof.~Piyush
%Srivastava for bringing the reference \citet{lovasz2004hit} to our
%attention.

\section{AI Disclosure}
The authors used GPT-6 Astra during the preparation of this paper to assist with proofreading and polishing the exposition. In addition, GPT-6 Astra suggested the second term in the potential function \eqref{eq:potential}, which the authors initially found nontrivial.
%The authors interacted with GPT 6 Astra during the preparation of this paper to assist proofread and polish the exposition. Furthermore, the second term in the potential function \eqref{eq:potential} was suggested by GPT 6 Astra, which appeared non-trivial to the authors at the first glance. 
\newpage